\documentclass[11pt]{article}
\usepackage{geometry}

\usepackage[utf8]{inputenc} % allow utf-8 input
\usepackage[T1]{fontenc}    % use 8-bit T1 fonts
\usepackage[pagebackref=true]{hyperref}       % hyperlinks
\usepackage{url}            % simple URL typesetting
\usepackage{booktabs}       % professional-quality tables
\usepackage{amsfonts}       % blackboard math symbols
\usepackage{nicefrac}       % compact symbols for 1/2, etc.
\usepackage{microtype}      % microtypography
\usepackage{xcolor}         % colors
\usepackage{amsmath}
\usepackage{algorithm}
\usepackage{algorithmic}
\usepackage{bbm}
\usepackage{subfig}
\usepackage{graphicx}
\usepackage{mathtools}
\usepackage{pdflscape}
\usepackage{listings}
\usepackage{caption}
\usepackage{amsthm}
\usepackage{enumitem}
\usepackage{authblk}
\usepackage[round]{natbib}
\usepackage{mathrsfs}

\renewcommand*\backref[1]{\ifx#1\relax \else (Cited on #1) \fi}

\newcommand{\comm}[1]{}
\newcommand{\cX}{\mathcal{X}}
\newcommand{\cO}{\mathcal{O}}
\newcommand{\cH}{\mathcal{H}}
\newcommand{\cD}{\mathcal{D}}
\newcommand{\cS}{\mathcal{S}}
\newcommand{\cU}{\mathcal{U}}
\newcommand{\cB}{\mathcal{B}}
\newcommand{\cL}{\mathcal{L}}

\newcommand{\R}{\mathbb{R}}

\newcommand{\DUETS}{\textsc{DUETS }}

\DeclareMathOperator*{\argmax}{arg\,max}
\DeclareMathOperator*{\argmin}{arg\,min}

\newtheorem{theorem}{Theorem}[section]

\newtheorem{lemma}[theorem]{Lemma}

\theoremstyle{definition}
\newtheorem{definition}[theorem]{Definition}
\newtheorem{assumption}[theorem]{Assumption}

\theoremstyle{remark}

\title{Order-Optimal Regret in Distributed Kernel Bandits using Uniform Sampling with Shared Randomness}
\author[*]{Nikola Pavlovic}
\author[$\ddagger$]{Sudeep Salgia}
\author[*]{Qing Zhao}
\affil[*]{School of Electrical \& Computer Engineering, Cornell University, Ithaca, NY, \emph{\{np358, qz16\}@cornell.edu} }
\affil[$\ddagger$]{Carnegie Mellon University, Pittsburgh, PA, \emph{ssalgia@andrew.cmu.edu}}
\date{Feb 2024}

\begin{document}

\maketitle

\begin{abstract}
   We consider distributed kernel bandits where $N$ agents aim to collaboratively maximize an unknown reward function that lies in a reproducing kernel Hilbert space. Each agent sequentially queries the function to obtain noisy observations at the query points. Agents can share information through a central server, with the objective of minimizing regret that is accumulating over time $T$ and aggregating over agents. We develop the first algorithm that achieves the optimal regret order (as defined by centralized learning) with a communication cost that is sublinear in both $N$ and $T$. The key features of the proposed algorithm are the uniform exploration at the local agents and shared randomness with the central server. Working together with the sparse approximation of the GP model, these two key components make it possible to preserve the learning rate of the centralized setting at a diminishing rate of communication.  
\end{abstract}

\section{Introduction}

\subsection{Distributed Kernel Bandits}

We study the problem of zeroth-order online stochastic optimization in a distributed setting, where $N$ agents aim to collaboratively maximize a reward function with communications facilitated by a central server. The reward function $f:\cX \rightarrow \mathbb{R}$ is unknown; it is only known that it lives in a Reproducing Kernel Hilbert Space (RKHS) associated with a known kernel $k$. Each agent sequentially chooses points in the function domain $\cX$ to query and subsequently receives noisy feedback on the function values (i.e., random rewards) at the query points. The goal is for each distributed agent to converge quickly to $x^{*} \in \argmax_{x \in \cX} f(x)$, a global maximizer of $f$. We quantify this goal as minimizing the cumulative regret summed over a learning horizon of length $T$ and over all $N$ agents:
\begin{align}
R=\sum_{n=1}^{N}\sum_{t=1}^{T} \left(f(x^*)-f(x^{(n)}_{t})\right),
\label{eqn:regret_definition}
\end{align}
where $x^{(n)}_{t}$ denotes the point queried by agent $n$ at time $t$.  \\

The above zeroth-order stochastic optimization problem can be viewed as a continuum-armed kernelized-bandit problem~\citep{Srinivas}. The expressive power of the RKHS model represents a broad family of objective functions. In particular, it is known that the RKHS of typical kernels, such as the
Mat\'ern family of kernels, can approximate almost all continuous functions on compact subsets of $\R^d$~\citep{Srinivas}. The problem has been studied extensively under a centralized setting with a single decision maker (i.e., $N=1$), for which several algorithms have been proposed, including  UCB-based algorithms~\citep{Srinivas, Gopalan,yadkori_Linear}, batched pure exploration~\citep{Batched_Communication}, tree-based domain shrinking~\citep{GP_ThreDS} and RIPS~\citep{Camilleri2021RIPS}. Optimal learning efficiency in terms of regret order in $T$ has been obtained in both the stochastic~\citep{Batched_Communication, GP_ThreDS} and the contextual setting~\citep{Valko_Kernel}. \\

In addition to learning efficiency, distributed kernel bandits face a new challenge of communication efficiency. Without constraints on the communication overhead, all agents can share their local observations and coordinate their individual query actions at no cost. The distributed problem can be trivially reduced to a centralized one. At the other end of the spectrum is a complete decoupling of the agents, resulting in $N$ independent single-user problems without the benefit of data sharing for accelerated learning. The tension between learning efficiency (which demands data sharing and action coordination) and communication efficiency is evident. A central question to this trade-off is how to achieve the optimal learning rate enjoyed by the centralized setting using a minimum amount of message exchange among agents. \\

In contrast to the extensive literature on centralized kernel bandits, distributed kernel bandits are much less explored despite their broad applications (e.g., federated learning for hyperparameter tuning~\citep{Dai2020FBO} and collaborative training of neural nets using the recent theory of Neural Tangent Kernel~\citep{jacot2018neural}). There exist only a handful of studies under drastically different settings and constraints (see Sec.\ref{sec:Related_Work}). For the setting considered in this work, no distributed learning algorithms exist that achieve the optimal regret order with a sublinear (in both $T$ and $N$) message exchange among agents.

\subsection{Main Results}

In this paper, we develop the first algorithm for distributed kernel bandits that achieves the optimal order of regret enjoyed by centralized learning with a sublinear message exchange in both $T$ and $N$. \\

To tackle the essential tradeoff between learning rate and communication efficiency, a distributed learning algorithm needs a communication strategy that governs \emph{what} to communicate and \emph{how to integrate} the shared information into local query actions. To minimize the total regret that is accumulating over time and aggregating over the agents, the communication strategy needs to work in tandem with the query actions to ensure a continual flow of information available at all agents for decision-making.\\

A natural answer to \emph{what} to communicate in a distributed learning problem is certain sufficient local statistics of the underlying unknown parameters. For example, for multi-armed (i.e., discrete arms) and linear bandits, this corresponds to the local estimates of the arm mean values and the mean reward vector respectively. However, for kernel bandits, the corresponding quantity would be an estimate of the function, which is potentially infinite-dimensional and hence an impractical choice for communication. Existing studies resolve this issue by exchanging local query actions and observations across all agents and throughout the learning horizon~\citep{Approx_Dis, Dubey_Kernel_Distributed}, resulting in a communication cost growing linearly in both $N$ and $T$. \\

Even with a communication cost growing linearly in both $N$ and $T$, preserving  the full learning power of a centralized decision maker with $NT$ query points is not immediate. The prevailing approaches to centralized kernel bandits that achieve order optimal regret build on the maximum posterior variance (MPV) sampling strategy~\citep{Batched_Communication} which queries, at each time, the point with the highest posterior variance conditioned on all past observations. Ensuring such a maximal uncertainty reduction at each query point is believed to be crucial in utilizing the full statistical power of all query points. Unfortunately, such a fully adaptive query strategy is incompatible with the parallel learning among distributed agents. To emulate the MPV sampling at each of the $NT$ query points would require the agents to \emph{take turns} in their queries and share the local observations immediately with all other agents, an infeasible strategy for most distributed learning problems.  Implementing MPV-based sampling in parallel across agents, however, loses the full adaptivity. This is arguably the main obstacle in realizing the optimal learning rate of a centralized kernel bandit in a distributed setting.  \\

To tackle the above challenges, our proposed algorithm represents major departures from the prevailing approaches. Referred to as \textsc{DUETS} (Distributed Uniform Exploration of Trimmed Sets), this algorithm has two key features: \emph{uniform exploration} at the local agents and \emph{shared randomness} with the central server.  \\

In \DUETS, each agent employs uniform (at random) sampling as the query strategy. Uniform sampling is fully compatible with parallel learning. In particular, note that the union of the local sets of size $t$ query points obtained at the agents through uniform sampling is identical (in distribution) to the set of size $Nt$ query points obtained at a centralized decision maker using the same uniform sampling strategy. This superposition property of uniform sampling allows us to leverage the recent results on random exploration in centralized kernel bandits~\citep{Sudeep_Uniform_Sampling}, and is crucial in achieving the optimal learning rate defined by the centralized setting. In addition to preserving the learning rate of the centralized setting, uniform sampling enjoys advantages in computation as well as communication aspects. Comparing with the MPV strategy that requires an expensive maximization of a non-convex acquisition function for finding each query point, uniform sampling is extremely simple to implement. This computational efficiency can be particularly attractive to distributed local devices. In terms of communication efficiency, uniform sampling makes it possible to bypass the exchange of query points altogether and reduce the exchange of reward observations through the \emph{shared randomness} strategy detailed below.  \\

In DUETS, each agent has access to an independent coin, i.e., a source of randomness, which is unknown to the other agents but is known to the server. The shared randomness  enables the server to reproduce the points queried by the agents, thereby resulting in effective transmission of the local set of queried points at each agent to the server at \emph{no communication cost}. To reduce the communication overhead associated with the reward observations, we employ sparse approximation of GP models~\citep{GP_RKHS_Approx_Equiv}. The availability of \emph{all} the queried points at the server provides the perfect platform for leveraging the power of sparse approximation to reduce the communication to a diminishing fraction of the total number of observations. Specifically, the server, with access to all the query points, selects a small subset of points that can approximate, to sufficient accuracy, the posterior statistics corresponding to all the points queried by the agents. This allows a diminishing rate of communication to share local reward observations. It is this integration of uniform sampling, shared randomness, and sparse approximation in DUETS that makes it possible to achieve the optimal learning rate of the centralized setting at a communication cost that is sublinear in both $N$ and $T$.  \\

We analyze the performance of \DUETS and establish that it incurs a cumulative regret of $\widetilde{\mathcal{O}}(\sqrt{NT\gamma_{NT}}\log(T/\delta))$\footnote{The notation $\tilde{\cO}(\cdot)$ hides poly-logarithmic factors.} with probability $1 - \delta$, where $\gamma_{NT}$ denotes the maximal information gain of the kernel and represents the effective dimension of the kernel. Note that this matches the lower bound (up to logarithmic factors) for any centralized algorithm with a total of $NT$ queries as established in~\citet{Scarlett2017LowerBoundGP}, thereby establishing the order-optimality of the proposed algorithm. To the best knowledge of the authors, this is the \emph{first} algorithm to achieve the optimal order of regret for the problem of distributed kernel bandits. We also establish a bound of $\tilde{\cO}(\gamma_{NT})$ on the communication cost incurred by \DUETS, where communication cost is measured by the number of real numbers transmitted during the algorithm (See Section~\ref{sec:problem_formulation} for more details). This significantly improves over the state-of-the-art of $\cO(N\gamma_{NT}^3)$ achieved by ApproxDisKernelUCB algorithm proposed by~\citet{Approx_Dis} and is always guaranteed to sublinear in the total number of queries, $NT$.

\subsection{Related Work}
\label{sec:Related_Work}

The existing literature on distributed kernel bandits is relatively slim. The most relevant to our work is that by~\citet{Approx_Dis}, where the authors consider the problem of distributed contextual kernel bandits and propose a UCB based policy with sparse approximation of GP models and intermittent communication. Their proposed policy was shown to incur a cumulative regret of $\widetilde{\mathcal{O}}(\sqrt{NT}\gamma_{NT})$ and communication cost of $\mathcal{O}(N\gamma_{NT}^3)$. The \DUETS algorithm proposed in this work, offers an improvement over the algorithm in~\citet{Approx_Dis} both in terms of regret and communication cost. While the contextual setting with varying arm action sets considered in their work is more general that the setting with a fixed arm set considered in this work, their proposed algorithm does not offer non-trivial reduction in regret or communication cost in the fixed arm setting. Moreover, both the regret and communication cost incurred by the algorithm in~\citet{Approx_Dis} are not guaranteed to be sublinear in the total number of queries, $NT$, for all kernels. Consequently, their algorithm does not guarantee convergence to $x^*$ or a non-trivial communication cost for all kernels. On the other hand, both regret and communication cost of \DUETS is guaranteed to be sub-linear implying both convergence and communication efficiency. \\

Among other studies,~\citet{Distributed_Simple_Regret} consider the problem of distributed pure exploration in kernel bandits over finite action set, where they focus on designing learning strategies with low simple regret. In this work, we consider the more challenging continuum-armed setup with a focus on minimizing cumulative regret as opposed to simple regret. Another line of work explores impact of heterogeneity among clients and design algorithms to minimize this impact.~\citet{Vakili_Suddep_Sketching} consider personalized kernel bandits in which agents have heterogeneous models and aim to optimize the weighted sum of their own reward function and the average reward function over all the agents.~\citet{Dubey_Kernel_Distributed} consider heterogeneous distributed kernel bandits over a graph in which they use additional kernel-based modeling to measure task similarity across different agents.  \\

In contrast to the distributed kernel bandit, the problems of distributed multi-armed bandits and linear bandits have been extensively studied. For distributed multi-armed bandits (MAB), a variety of algorithms have been proposed for distributed learning under different network topologies~\citep{Landgren2017, Shahrampour2017, Sankararaman2019, Chawla2020gossip, Zhu2021}.~\citet{Shi2021} and~\citet{Fed_MAB} have analyzed the impact of heterogeneity among agents in the distributed MAB problem. Similarly, the problem of distributed linear bandits is also well-understood in variety of settings with different network topologies~\citep{Korda2016}, heterogeneity among agents~\citep{Mitra2021, Ghosh2021, Hanna2022ContextualLinearBandits} and communication constraints~\citep{Mitra2022bitconstrained, WANG_Delb_Demab, Huang2021, Amani2022, Sudeep_Linear}.

\section{Problem Formulation}\label{sec:problem_formulation}

We consider a distributed learning framework consisting of $N$ agents indexed by $\{1,2,\dots, N\}$. Under this framework, we study the problem of collaboratively maximizing an unknown function $f: \mathcal{X} \to \R$, where $\cX \subset \R^d$ is a compact, convex set. The function $f$ belongs to the Reproducing Kernel Hilbert Space (RKHS), $\cH_k$, associated with a known positive definite kernel $k : \cX \times \cX \to \R$. The RKHS, $\cH_k$, is a Hilbert space that is endowed by with an inner product $\langle \cdot, \cdot \rangle_{\cH_k}$ that obeys the reproducing property, i.e., $\langle g, k(x, \cdot) \rangle_{\cH_k} = g(x)$ for all $g \in \cH_k$, and induces the norm $\|g\|_{\cH_k} = \langle g, g \rangle_{\cH_k}$. \\

The agents can access the unknown function by querying the function at different points in the domain $\cX$. Upon querying a point $x \in \cX$, the agent receives a reward $y = f(x) + \epsilon$, where $\epsilon$ is a noise term. We make the following assumptions on the unknown function $f$ and noise.

\begin{assumption}\label{bounded_RKHS_norm}
 The RKHS norm of the function $f$ is bounded by a known constant $B$, i.e., $\|f\|_{\cH_k} \leq B$.
\end{assumption}

\begin{assumption}\label{Sub_Gaussian}
The noise term $\epsilon$ is assumed to be independent across all agents and all queries and is a zero-mean, $R$ sub-Gaussian random variable i.e.,  it satisfies the relation $\mathbb{E}[\exp(\lambda \epsilon)]\leq \exp{\frac{\lambda^2R^2}{2}}$ for all $\lambda \in \mathbb{R}$.
\end{assumption}

\begin{assumption}\label{Grid}
For each $r \in \mathbb{N}$, there exists a discretization $\cU_r$ of $\mathcal{X}$ with $|\cU_r| = \mathrm{poly}(r)$\footnote{The notation $g(x) = \mathrm{poly}(x)$ is equivalent to $g(x) = \cO(x^k)$ for some $k \in \mathbb{N}$.} such that, for any $f \in \mathcal{H}_k$, we have $|f(x)-f([x]_{\cU_r})|\leq \frac{\|f\|_{\mathcal{H}_k}}{r}$, where $[x]_{\cU_r} =\argmin_{x'\in \cU_r} \|x-x'\|_2$. 
% We can also bound the size of $\cD_r$ as $\|\mathcal{X}_M\|<cB^dM^d$     
\end{assumption}

\begin{assumption}
    Let $\cL_{\eta} = \{x \in \cX | f(x) \geq \eta \}$ denote the level set of $f$ for $\eta \in [-B, B]$. We assume that for all $\eta \in [-B,B]$, $\cL_{\eta}$ is a disjoint union of at most $M_f < \infty$ components, each of which is closed and connected. Moreover, for each such component, there exists a bi-Lipschitzian map between each such component and $\cX$ with normalized Lipschitz constant pair $L_f, L_f' < \infty$.
    \label{assumption:f_level_set_regularity}
\end{assumption}

Assumptions~\ref{bounded_RKHS_norm}-\ref{Grid} are standard, mild assumptions that are commonly adopted in the literature~\citep{Srinivas, Gopalan, Batched_Communication, Vakili_Aprox_Conv, Vakili_Kernel_Simple_Regret}. The existence of the discretization $\cU_r$ in Assumption~\ref{Grid} has been justified and adopted in previous studies~\citep{Srinivas, Vakili_Kernel_Simple_Regret}. In particular, the popular class of kernels like Squared Exponential and Mat\'ern kernels are known to be Lipschitz continuous, in which case a $\varepsilon$-cover of the domain with $\varepsilon = \cO(1/r)$ is sufficient to show the existence of such a discretization. At a high level, Assumption~\ref{assumption:f_level_set_regularity} ensures that the structure of the levels sets of $f$ satisfy a mild regularity condition. This is a mild assumption on $f$ that we require to adopt a result from~\citet{Sudeep_Uniform_Sampling} for our analysis. \\

The agents collaborate with each other by communicating through a central server. At each time instant, each agent can send a message to the server through the uplink channel. Based on the messages from different agents received by the server, it can then broadcast a message back to all the agents through the downlink channel. \\

Our objective is to design a distributed learning policy $\pi$ that specifies for each agent $n$, the point $x_{t}^{(n)}$ to be queried at each time instant $t$, based on the information available at that agent upto time instant $t$. The performance of a collaborative learning policy $\pi$ is measured through its performance in terms of both learning and communication efficiency over a learning horizon of $T$ steps. The learning efficiency is measured using the notion of cumulative regret, as defined in~\eqref{eqn:regret_definition}. \\

The communication efficiency is measured using the sum of the uplink and downlink communication costs. In particular, let $C_{\mathrm{up}}^{(n)}(T)$ denote the number of real numbers sent by the agent $n$ to the server over the time horizon. The uplink cost of $\pi$, $C_{\mathrm{up}}^{\pi}(T)$ is then given as the average communication cost over all agents:
\begin{align}
    C_{\mathrm{up}}^{\pi}(T) = \frac{1}{N}\sum_{n = 1}^N C_{\mathrm{up}}^{(n)}(T).
\end{align}
Similarly, the downlink cost of $\pi$, $C_{\mathrm{down}}^{\pi}(T)$ is given as the number of real numbers broadcast by the server over the entire time horizon averaged over all agents . The overall communication cost of $\pi$, $C^{\pi}(T)$, is given as $C^{\pi}(T) =  C_{\mathrm{up}}^{\pi}(T) + C_{\mathrm{down}}^{\pi}(T)$. \\

The objective is to design a distributed learning policy that achieves the order-optimal cumulative regret and incurs a low communication cost. We aim to provide high probability bounds on both the cumulative regret and communication cost that hold with probability $1-\delta$ for any given $\delta \in (0,1)$. \\

We overview the basis of Gaussian Process models and their sparse approximation, both of which are central to our proposed policy.

\subsection{ GP Models}

In this section we present a brief overview of Gaussian Process models and their application on establishing confidence interval for RKHS elements. \\

A Gaussian Process (GP) is a random process $G$ indexed by $\cX$ and is associated with a mean function $\mu : \cX \to \R$ and a positive definite kernel $k : \cX \times \cX \to \R$. The random process $G$ is defined such that for all finite subsets of $\cX$,  $\{x_1, x_2, \dots, x_m\} \subset \cX$, $m \in \mathbb{N}$, the random vector $[G(x_1), G(x_2), \dots, G(x_m)]^{\top}$ follows a multivariate Gaussian distribution with mean vector $[\mu(x_1), \dots, \mu(x_n)]]^{\top}$ and covariance matrix $\Sigma = [k(x_i, x_j)]_{i,j=1}^m$. Throughout the work, we consider GPs with $\mu \equiv 0$. When used as a prior for a data generating process under Gaussian noise, the conjugate property provides closed form expressions for the posterior mean and covariance of the GP model. Specifically, given a set of observations $\{\mathbf{X}_m,\mathbf{Y}_m\} = \{(x_i,y_i)\}_{i=1}^m$ from the underlying process, the expression for posterior mean and variance of GP model is given as follows:
\begin{align}
    \mu_{m}(x) & =k_{\mathbf{X}_m}(x)^{\top}(\lambda \mathbf{I}_m+\mathbf{K}_{\mathbf{X}_m,\mathbf{X}_m})^{-1}\mathbf{Y}_m, \label{eqn:posterior_mean}\\
    \sigma^2_m(x)& =(k(x,x)- k_{\mathbf{X}_m}^{\top}(x)(\lambda \mathbf{I}_m+\mathbf{K}_{\mathbf{X}_m,\mathbf{X}_m})^{-1}k_{\mathbf{X}_m}(x)). \label{eqn:posterior_variance}
\end{align}
In the above expressions, $k_{\mathbf{X}_m}(x)=[k(x_1,x),k(x_2,x)\dots k(x_n,x)]^{\top}$, $\mathbf{K}_{\mathbf{X}_m,\mathbf{X}_m}=\{k(x_i,x_j)\}_{i,j=1}^{m}$, $\mathbf{I}_m$ is the $m \times m$ identity matrix and $\lambda$ corresponds to the variance of the Gaussian noise. \\

Following a standard approach in the literature~\citep{Srinivas}, we model the data corresponding to observations from the unknown $f$, which belongs to the RKHS of a positive definite kernel $k$, using a GP with the same covariance kernel $k$. In particular, we assume a \emph{fictitious} GP prior over the fixed, unknown function $f$ along with \emph{fictitious} Gaussian distribution for the noise. The benefit of this approach is that the posterior mean and variance of this GP model serve as tools to both predict the values of the function $f$ and quantify the uncertainty of the prediction at unseen points in the domain, as shown by the following lemma .
\begin{lemma}{~\citet[Thm. 1]{Vakili_Kernel_Simple_Regret}}
Assume that~\ref{bounded_RKHS_norm} and~\ref{Sub_Gaussian} hold. Given a set of observations $\{\mathbf{X}_m,\mathbf{Y}_m\}$ as described above, such that the query points $\mathbf{X}_m$ are chosen independent of the noise sequence, then for a fixed $x\in \mathcal{X}$, the following relation holds with probability at least $1- \delta$:
\[
|h(x)-\mu_m(x)|\leq \beta(\delta) \cdot \sigma_m(x),
\]
where $\beta(\delta) =B+R\sqrt{\frac{2}{\lambda}\log{\left(\frac{2}{\delta}\right)}}$.
\end{lemma}

We would like to emphasize that these assumptions are modeling techniques used as a part of algorithm and not a part of the problem setup. In particular, the function $f$ is \emph{fixed, deterministic} function in $\cH_k$ and the noise is $R$-sub-Gaussian. \\

\label{gamma_explained}
Lastly, given a set of points $\mathbf{X}_m = \{x_1, x_2, \dots, x_m\} \in \cX$, the information gain of the set $\mathbf{X}_m$ is defined as $\gamma_{\mathbf{X}_m} := \frac{1}{2} \log(\det(\mathbf{I}_m + \lambda^{-1}\mathbf{K}_{\mathbf{X}_m,\mathbf{X}_m}))$. Using this, we can define the maximal information gain of a kernel as $\gamma_m := \sup_{\mathbf{X}_m} \gamma_{\mathbf{X}_m}$. Maximal information gain is closely related to the effective dimension of a kernel~\citep{Calandriello_Sketching} and helps characterize the regret performance of kernel bandit algorithms~\citep{Srinivas,Gopalan}. $\gamma_m$ depends only the kernel and $\lambda$ and has been shown to be an increasing sublinear function of $m$~\citep{Srinivas, information_gain_bound}.

\subsection{Sparse approximation of GP models}

The sparsification of GP models refers to the idea of approximating the posterior mean and variance of a GP model, corresponding to a set of observations $\{\mathbf{X}_m, \mathbf{Y}_m\}$, using a subset of query points $\mathbf{X}_m$. In particular, let $\cS$ be a subset of $\mathbf{X}_m$ consisting of $r < m$ points. The approximate posterior mean and variance based on the points in $\cS$, referred to as the inducing set, is given as follows\citep{GP_RKHS_Approx_Equiv}.
\begin{align}
	\tilde{\mu}_m(x) &= z_{\mathcal{S}}(x)^\top \left(\lambda\mathbf{I}_{|\mathcal{S}|}+\mathbf{Z}_{\mathbf{X}_m,\mathcal{S}}^\top\mathbf{Z}_{\mathbf{X}_m, \mathcal{S}}\right)^{-1}\mathbf{Z}^{\top}_{\mathcal{X}_m, \mathcal{S}}\mathbf{Y}_m  \label{eq:aprox_post_mean} \\
	\lambda\tilde{\sigma}^2_m(x) & = \big[k(x,x)- z_{\cS}^{\top} (x)\mathbf{Z}^\top_{\mathbf{X}_m,\mathcal{S}}\mathbf{Z}_{\mathbf{X}_m,\mathcal{S}}\left(\lambda\mathbf{I}_{|\mathcal{S}|}+\mathbf{Z}_{\mathbf{X}_m,\mathcal{S}}^\top\mathbf{Z}_{\mathbf{X}_m, \mathcal{S}}\right)^{-1}z_{\cS}(x)\big],\label{eq:aprox_post_variance}
\end{align}
where $z_{\mathcal{S}}(x)=\mathbf{K}_{\mathcal{S},\mathcal{S}}^{-\frac{1}{2}}k_{\mathcal{S}}(x)$ and $\mathbf{Z}_{\mathbf{X}_m, \mathcal{S}} = [z_{\cS}(x_1), z_{\cS}(x_2), \dots, z_{\cS}(x_m)]^{\top}$. \\

Note that it is sufficient to know the matrix $\mathbf{Z}_{\mathcal{X}_m,\mathcal{S}}^\top\mathbf{Z}_{\mathcal{X}_m,{\mathcal{S}}} \in \R^{r \times r}$, vector $\mathbf{Z}^{\top}_{\mathcal{X}_m, \mathcal{S}}\mathbf{Y}_m \in \R^r$ and the set $\mathcal{S}$ in order for $\tilde{\mu}$ and $\tilde{\sigma}$ to be calculated.

\section{The \textsc{DUETS} Algorithm}
\label{sec:Algorithm}

In this section, we present the proposed algorithm DUETS.  \\

We first describe the randomization at each agent and the shared randomness with the server. Each agent $n$ has a private coin $\mathscr{C}_n$ for generating random bits that are independent of those generated by other agents. Each agent's coin is private to other agents, but known to the central server. As a result, the server can reproduce the random bits generated at all agents.  \\

DUETS employs an epoch-based elimination structure where the domain $\cX$ is successively trimmed across epochs to maintain an active region that contains a global maximizer $x^*$ with high probability for future exploration. Specifically, in each epoch $j$, the server and the agents maintain a common active subset of the domain $\cX_{j} \subseteq \cX$  with $\cX_1$ initialized to $\cX$. The operations in each epoch are as follows.  \\

During the $j^{\text{th}}$ epoch, each agent $n$, using its private coin $\mathscr{C}_n$, generates $\cD_j^{(n)}$, a set of $T_j$ points that are uniformly distributed in the set $\cX_{j}$\footnote{If the active region consists of multiple disjoint regions, then we carry out this step for each region separately. For simplicity of description, we assume the active region consists of a single connected component.}. $T_j$ is set to $\lfloor \sqrt{TT_{j-1}} \rfloor$, with $T_1$ being an input to the algorithm. Each agent $n$ queries all the points in $\cD_j^{(n)}$ and obtains $\mathbf{Y}_j^{(n)} \in \R^{T_j}$, the corresponding vector of reward observations. \\

Since the server has access to the coins of all the agents, it can faithfully reproduce the set $\cD_j = \bigcup_{n = 1}^{N} \cD_j^{(n)}$ without any communication between the server and the agents. In order to efficiently communicate the observed reward values from the agents to the server, we leverage sparse approximation of GP models along with the knowledge of the set $\cD_j$ at the server. The server constructs a global inducing set $\cS_j$ by including each point in $\cD_j$ with probability $p_j := p_0 \sigma_{j,\max}^2$, independent of other points where $\sigma_{j,\max}^2=\sup_{x \in \mathcal{X}_{j}} \sigma_j^2(x)$ and $\sigma_j^2(\cdot)$ is the posterior variance corresponding to points collected in $\mathcal{D}_j$. Here, $p_0=72\log\left(\frac{4NT}{\delta'}\right)$ is an appropriately chosen universal constant which ensures that the approximate posterior statistics constructed using $\cS_j$ are a faithful representation of the true posterior statistics corresponding to the set $\cD_j$ with probability $1-\delta$. The server broadcasts the inducing set $\cS_j$ to all the agents. \\

Upon receiving the inducing set, each agent $n$ computes the projection $v_j^{(n)} \in \R^{|\cS_j|}$ of its reward vector onto the inducing set as follows:
\begin{align}
    v_j^{(n)} :=  \mathbf{Z}_{\mathcal{D}^{(n)}_j,\mathcal{S}_j}^{\top}\mathbf{Y}^{(n)}_j.
\end{align}
Each agent then sends back the lower-dimensional projected observations $v_j^{(n)}$ to the server, which subsequently aggregates them to obtain the vector $\overline{v}_j$ given as
\begin{align}
    \overline{v}_j := \left(\lambda\mathbf{I}_{|\mathcal{S}_j|}+\mathbf{Z}_{\cD_j,\mathcal{S}_j}^\top \mathbf{Z}_{\cD_j, \mathcal{S}_j}\right)^{-1}\left( \sum_{n = 1}^N v_j^{(n)} \right).
\end{align}
Note that the summation $\sum_{n = 1}^N v_j^{(n)}$ equals to $ \mathbf{Z}_{\mathcal{D}_j,\mathcal{S}_j}^{\top}\mathbf{Y}_j$, i.e., projection of the rewards of all agents onto the inducing set. The server then broadcasts the vector $\overline{v}_j$ and $\sigma_{j, \max}$ to all the agents. The benefit of sending $\overline{v}_j$ as opposed to the sum of rewards is that it allows the agents to compute the posterior mean at the agents using their knowledge of the inducing set $\cS_j$ (See. Eqn~\eqref{eq:aprox_post_mean}). \\

As the last step of the epoch, all the agents and the server trim the current set $\cX_{j}$ to $\cX_{j+1}$ using the following update rule:
\begin{align}
    \cX_{j+1} = \left\{ x \in \cX_{j}: \tilde{\mu}_j(x) \geq \sup_{x' \in \cX_{j}} \tilde{\mu}_j(x') - 2\beta(\delta^{'}) \sigma_{j, \max}  \right\},
    \label{eqn:update_rule}
\end{align}
where $ \delta' = \frac{\delta}{2|\mathcal{U}_{T}| \cdot (\log(\log{N}\log{T}))+4)}$ and $\tilde{\mu}_j(x) = z^{\top}_{\cS_j}(x) \overline{v}_j$  is the \emph{approximate} posterior mean computed based on the inducing set $\cS_j$ (See Eqn.~\eqref{eq:aprox_post_mean}). Since the posterior mean provides an estimate for the function values, the update condition is designed to eliminate all points at which the (estimated) function value is smaller than the current best estimate of the maximum value, upto an estimation error. Note that trimming is a deterministic procedure which ensures that all the agents and the server share a common value of $\cX_{j+1}$. \\

A detailed pseudocode of both the agent and the server side of the DUETS is provided in Algorithms~\ref{alg:DISUS_Agent} and~\ref{alg:DISUS_Server} respectively.

\begin{algorithm}
\caption{\DUETS:  Agent $n \in \{1,2,\dots, N\}$}\label{alg:DISUS_Agent}
\begin{algorithmic}[1]
    \STATE \textbf{Input}: Size of the first epoch $T_1$, error probability $\delta$
    \STATE $t \leftarrow 0, j\leftarrow 1$, $\cX_1 \leftarrow \cX$ 
    \WHILE{$t < T$}
        \STATE $\cD_j^{(n)} = \emptyset$
        \FOR {$i \in \{1,2,\dots, T_j\}$}
            \STATE Query a point $x_t^{(n)}$ uniformly at random from $\cX_{j-1}$ using the coin $\mathscr{C}_n$ and observe $y_t^{(n)}$
            \STATE $\cD_j^{(n)} \leftarrow \cD_j^{(n)} \cup \{x_t^{(n)}\} $
            \STATE $t \leftarrow t + 1$
            \IF{$t > T$}
            \STATE \textbf{Terminate}
            \ENDIF
        \ENDFOR
        \STATE Receive the global inducing set $\mathcal{S}_j$
        \STATE Set $v_j^{(n)} \leftarrow  \mathbf{Z}_{\mathcal{D}^{(n)}_j,\mathcal{S}_j}^{\top}\mathbf{Y}^{(n)}_j$, where $\mathbf{Y}^{(n)}_j = [y_{t-T_j}, y_{t-T_j + 1}, \dots, y_t]^{\top}$
        \STATE Receive $\overline{v}_j$ and $\sigma_{j,\mathrm{max}}$ from the server
        \STATE Use $\overline{v}_j$ to compute $\tilde{\mu}_j(\cdot) = z^{\top}_{\cS_j}(\cdot) \overline{v}_j$
        \STATE Update $\cX_{j}$ to $\cX_{j+1}$ using Eqn.~\eqref{eqn:update_rule}
        \STATE $T_{j+1}\leftarrow \lfloor\sqrt{TT_{j}}\rfloor$
        \STATE $j \leftarrow j + 1$    
\ENDWHILE
\end{algorithmic}
\end{algorithm}

\begin{algorithm}
\caption{\DUETS: Server}\label{alg:DISUS_Server}
\begin{algorithmic}[1]
    \STATE \textbf{input}: Size of the first epoch $T_1$, error probability $\delta$
    \STATE $t \leftarrow 0, j\leftarrow 1$, $\cX_1 \leftarrow \cX$ 
    \WHILE{$t < T$}
        \STATE Use the coins $\mathscr{C}_1, \mathscr{C}_2, \dots, \mathscr{C}_N$ to reproduce the sets $\cD_j^{(1)}, \cD_j^{(2)}, \dots, \cD_j^{(N)}$
        \STATE $\cD_j \leftarrow \bigcup_{n =1}^N \cD_j^{(n)}$
        \STATE Set $\sigma_{j, \max} \leftarrow \sup_{x \in \cX_{j}} \sigma_{j}(x)$
        \STATE Construct the set $\cS_j$ by including each point from $\cD_j$ with probability $p_j$, independent of other points
        \STATE Broadcast $\cS_j$ to all the agents
        \STATE Receive $v_j^{(n)}$ from all agents $n \in \{1,2,\dots,N\}$
        \STATE Set $\overline{v}_j \leftarrow \left(\lambda\mathbf{I}_{|\mathcal{S}_j|}+\mathbf{Z}_{\cD_j^{(n)},\mathcal{S}_j}^\top \mathbf{Z}_{\cD_j^{(n)}, \mathcal{S}_j}\right)^{-1}(\sum_{n = 1}^N v_j^{(n)} ).$
        \STATE Broadcast $\overline{v}_j$ and $\sigma_{j, \max}$ to all the agents
        \STATE Update $\cX_{j}$ to $\cX_{j+1}$ using Eqn.~\eqref{eqn:update_rule}
        \STATE $t \leftarrow t + T_j$
        \STATE $T_{j+1}\leftarrow \lfloor\sqrt{TT_{j}}\rfloor$
        \STATE $j \leftarrow j + 1$ 
    \ENDWHILE
\end{algorithmic}
\end{algorithm}

\section{Performance Analysis}

The following theorem characterizes the regret performance and communication cost of DUETS.

\begin{theorem}\label{Main_Theorem}
    Consider the distributed kernel bandit problem described in Section~\ref{sec:problem_formulation}. For a given $\delta\in (0,1)$, let the policy parameters of \DUETS be such that $T_1 \geq \overline{M}/N$ and $p_0=72\log{\frac{4N}{\delta}}$. Then  with probability at least $1 - \delta$,  the regret and communication cost incurred by \DUETS satisfy the following relations:
    \begin{align*}
        R_{\mathrm{DUETS}} & = \tilde{\cO}(\sqrt{NT\gamma_{NT}}\log(T/\delta)) \\
        C_{\mathrm{DUETS}} & = \tilde{\cO}(\gamma_{NT}).
    \end{align*}
    Here, $\overline{M}$ is a constant that depends only upon the kernel $k$ and the domain $\cX$ and it is independent of $N$ and $T$.\footnote{The constant $\overline{M}$ is the same as one in Lemma~\ref{max_var_bound}, which has been adopted from~\citet{Sudeep_Uniform_Sampling}. We refer the reader to~\citet{Sudeep_Uniform_Sampling} for an exact expression of the constant and additional related discussion.}
\end{theorem}

As shown in above theorem, \DUETS achieves order-optimal regret as it matches the lower bound established in~\citet{Scarlett2017LowerBoundGP} upto logarithmic factors. \DUETS is the \emph{first algorithm} to close this gap to the lower bound in the distributed setup and achieve order-optimal regret performance. Moreover, \DUETS incurs a communication cost that is sublinear in both $T$ and $N$ for all kernels. Furthermore, it can be much smaller that $NT$, depending upon the smoothness of the kernel. For example, using the bounds on information gain~\citep{information_gain_bound}, we can show that the communication cost incurred by \DUETS is $\cO(\log^{d}(NT))$.

\begin{proof}
We provide a sketch of the proof of Theorem~\ref{Main_Theorem} here. The regret bound is obtained by first bounding the regret incurred by \DUETS in each epoch $j$ and then summing the regret across different epochs. In any epoch $j$, the agents take purely exploratory by uniformly sampling the region $\cX_j$. Thus, to bound the regret incurred at any step during an epoch, we use the crude bound $\Delta_j := \sup_{x \in \cX_{j}} (f(x^*) - f(x))$. Consequently, the regret during the $j^{\text{th}}$ epoch, denoted by $R^{(j)}$, is upper bounded by $N \cdot \Delta_j T_j$. Note that the update criterion (Eqn.~\eqref{eqn:update_rule}) is designed to obtain a refined localization of $x^*$ by eliminating the points with low function values consequently leading to smaller values of $\Delta_j$ as the algorithm proceeds. The epoch lengths are carefully chosen to balance the increase in epoch length with the decrease in $\Delta_j$ to obtain the tightest bound. These ideas are captured in the following lemmas from the regret bound follows.

\begin{lemma}\label{grid_bound} 
Let $\Delta_{j}: =\sup_{x\in \mathcal{X}_{j}}f(x^*)-f(x)$. Then, the following bound holds all epochs $j \geq 1$ with probability $1- \frac{\delta}{2}$. 
\begin{equation*}
\Delta_j\leq 8\beta(\delta') \cdot \left( \sup_{x\in \mathcal{X}_{j-1}}\sigma_{j}(x) \right) +\frac{4B}{T},
\end{equation*}   
where $\delta' = \frac{\delta}{2(\log(\log{N}+\log{T})+4) |\mathcal{U}_{T}|}$ and $\cU_T$ denotes the discretization defined in Assumption~\ref{Grid}.
\end{lemma}

\begin{lemma} %[\citet[Lemma 4.6]{Sudeep_Uniform_Sampling}]
    Let $\sigma_j^2(\cdot)$ denote the posterior variance corresponding to the set $\cD_j$ obtained by sampling $NT_j$ points uniformly at random from the domain $\cX_{j}$. Then, for $T_1 \geq \overline{M}(\delta)/N$ and for any $f$ satisfying Assumption~\ref{assumption:f_level_set_regularity}, the following bound holds with probability $1 - \delta$ for all epochs $j\geq 1$:
    \begin{align*}
        \sup_{x \in \cX_{j}} \sigma_j^2(x) \leq C_{f, \cX} \cdot \frac{\gamma_{NT_j}}{NT_j}.
    \end{align*}
    Here $C_f$ denotes a constant that depends only on $f$ and the domain $\cX$ and is independent of both $N$ and $T$.
    \label{max_var_bound}
\end{lemma}

\begin{lemma}\label{lemma:epoch_number}
The total number of epochs in \textsc{DUETS} over a time horizon of $T$ is less than $\log(\log(\max\{N,T\})) + 4$. %\\ 
\end{lemma} 

Lemma~\ref{max_var_bound} is result adopted from the recent work by~\citet{Sudeep_Uniform_Sampling} that establishes bounds on worst-case posterior variance corresponding to a set of randomly sampled points. \\

For the bound on communication cost, note that each epoch $j$, the server broadcasts the inducing set $\cS_j$, which consists of $|\cS_j|$ vectors in $\R^d$, the vector $\overline{v}_j \in \R^{|\cS_j|}$ and the scalar $\sigma_{j, \max}$, resulting in a downlink cost of $\cO(|\cS_j|)$ in epoch $j$. Similarly, since each agent just uploads ${v}_j^{(n)} \in \R^{|\cS_j|}$, the uplink cost of \DUETS in epoch $j$ also satisfies $\cO(|\cS_j|)$. Consequently, the communication cost of \DUETS in epoch $j$ is bounded by $\cO(|\cS_j|)$. The following lemma gives a high probability bound on the $|\cS_j|$. \\

\begin{lemma}
    Let $\cS_j$ denote the inducing set construct in $j^{\text{th}}$ epoch, as outlined in Section~\ref{sec:Algorithm}. Then, with probability at least $1 - \delta$, 
    \begin{align*}
        |\cS_j| \leq C_{f, \cX} \cdot \left(3+\log\left(\frac{\log(\log{N}\log{T})}{\delta}\right)\right) \cdot \gamma_{NT},
    \end{align*}
    holds for all epochs $j$. In the above expression, $C_{f, \cX}$ is same as the constant in Lemma~\ref{max_var_bound}.
    \label{lemma:inducing_set_size}
\end{lemma}

The bound on the communication cost follows directly from Lemmas~\ref{lemma:inducing_set_size} and~\ref{lemma:epoch_number}. Please refer to Appendix~\ref{sec: Appendix A} for a detailed proof.

\end{proof}

\begin{figure*}[h]
\centering
\subfloat[$h_1(x)$]{\label{fig:cosine_regret}\centering \includegraphics[scale = 0.2]{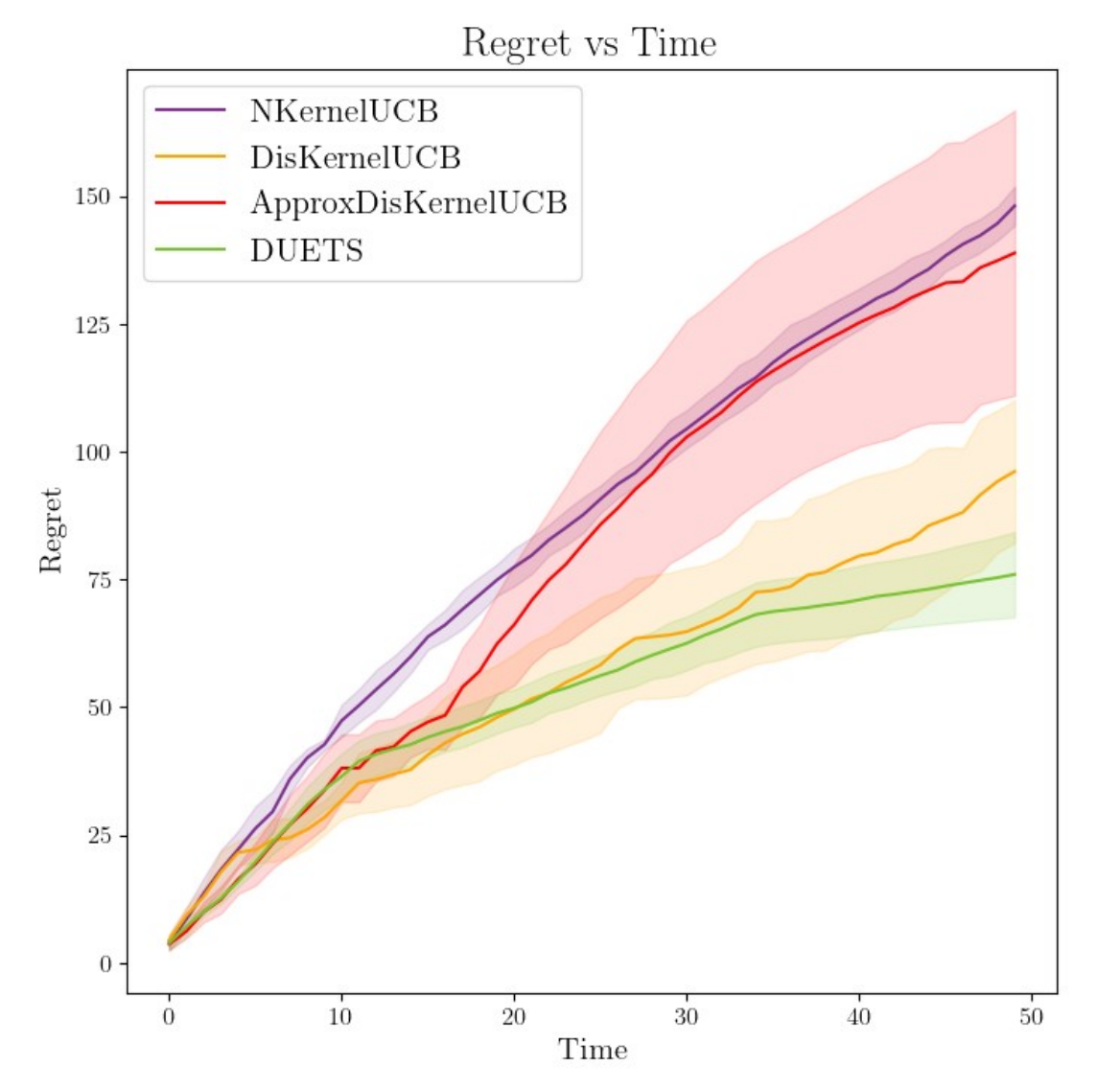}}
~
\subfloat[$h_2(x)$]{\label{fig:polynomial_regret}\centering \includegraphics[scale = 0.2]{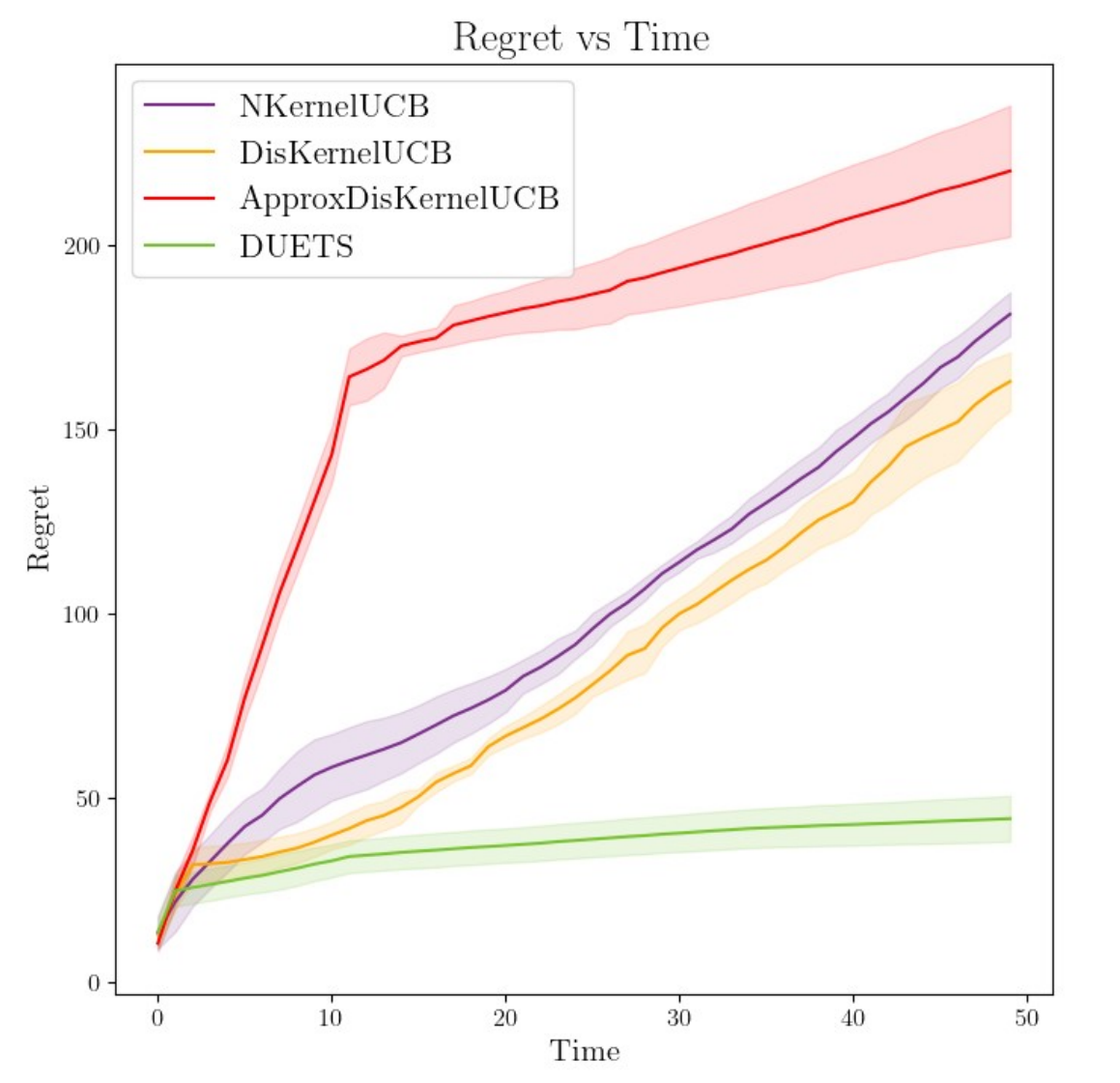}}
~
\subfloat[Branin]{\label{fig:branin_regret}\centering \includegraphics[scale = 0.2]{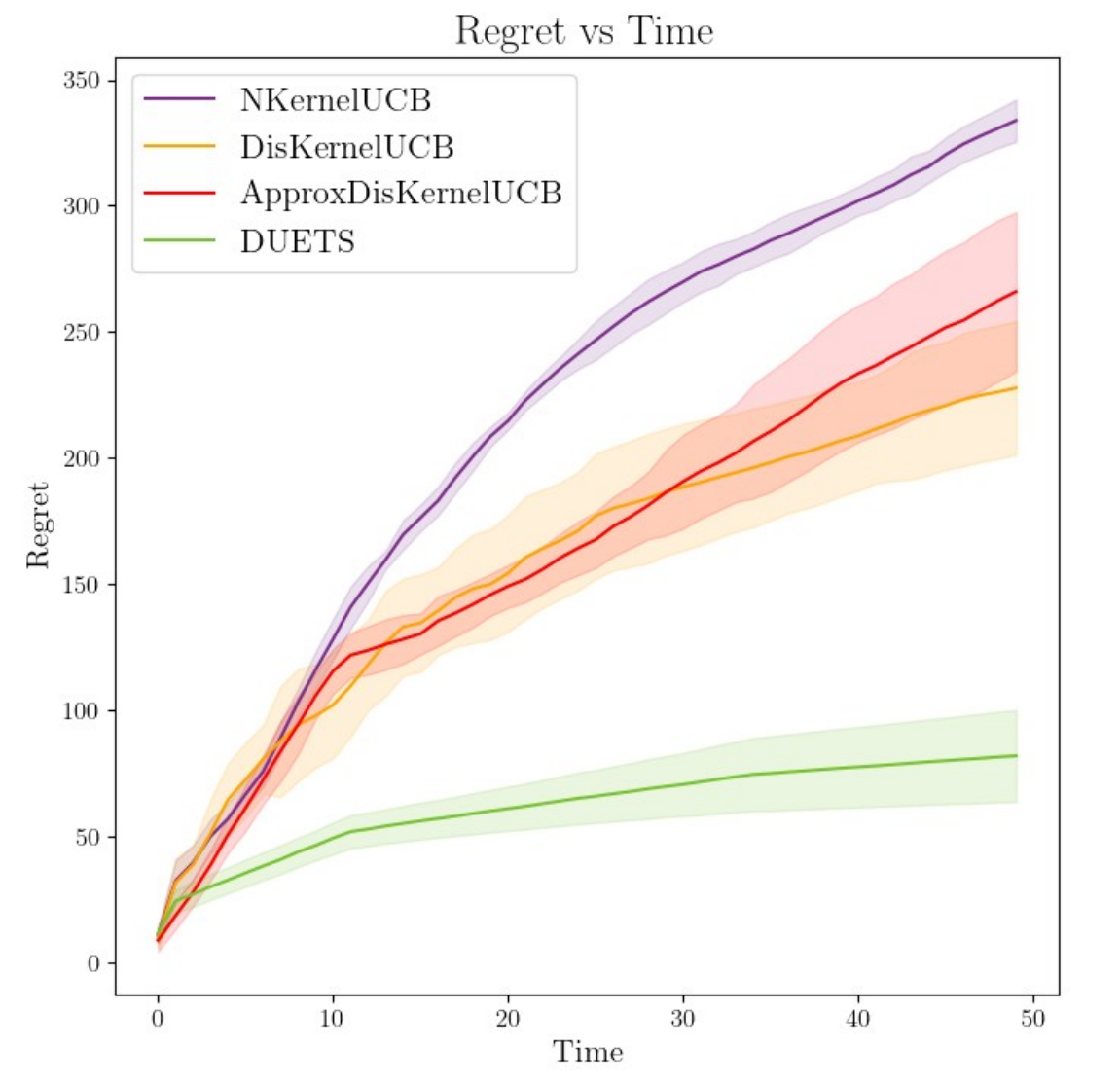}}
~
\subfloat[Hartmann-$4$D]{\label{fig:hartmann_regret}\centering \includegraphics[scale = 0.2]{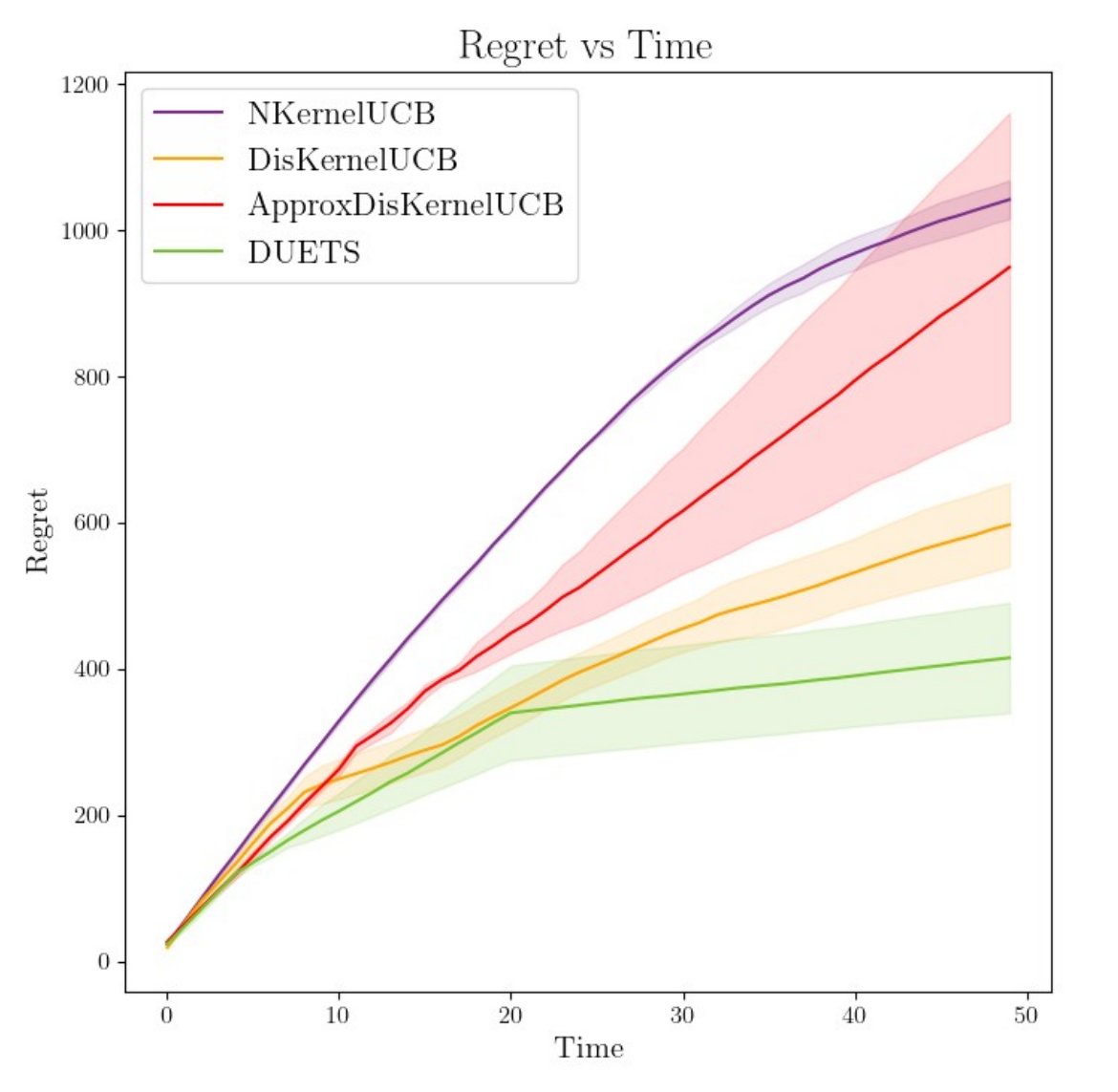}}

\subfloat[$h_1(x)$]{\label{fig:cosine_comm}\centering \includegraphics[scale = 0.2]{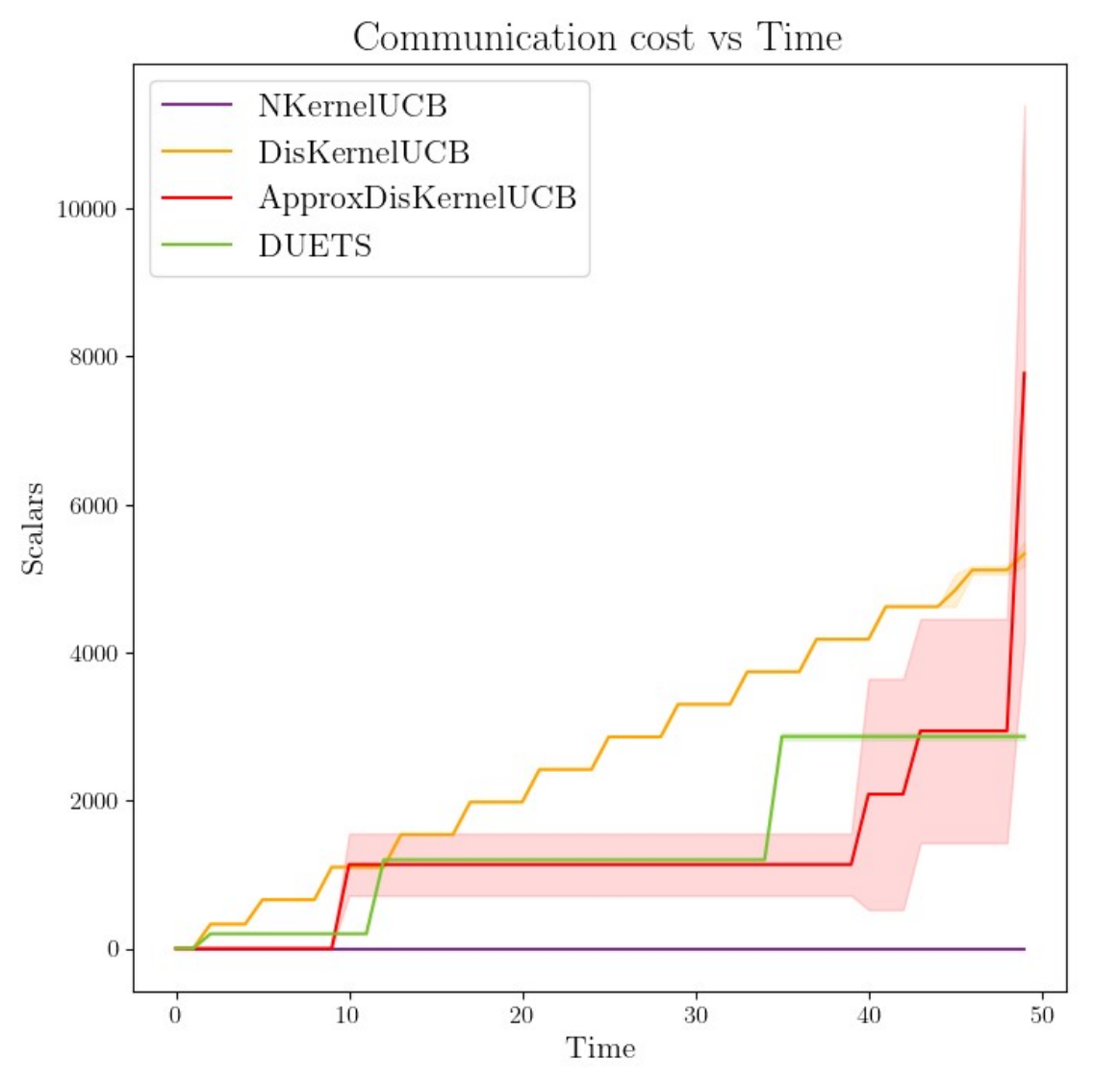}}
~
\subfloat[$h_2(x)$]{\label{fig:polynomial_comm}\centering \includegraphics[scale = 0.2]{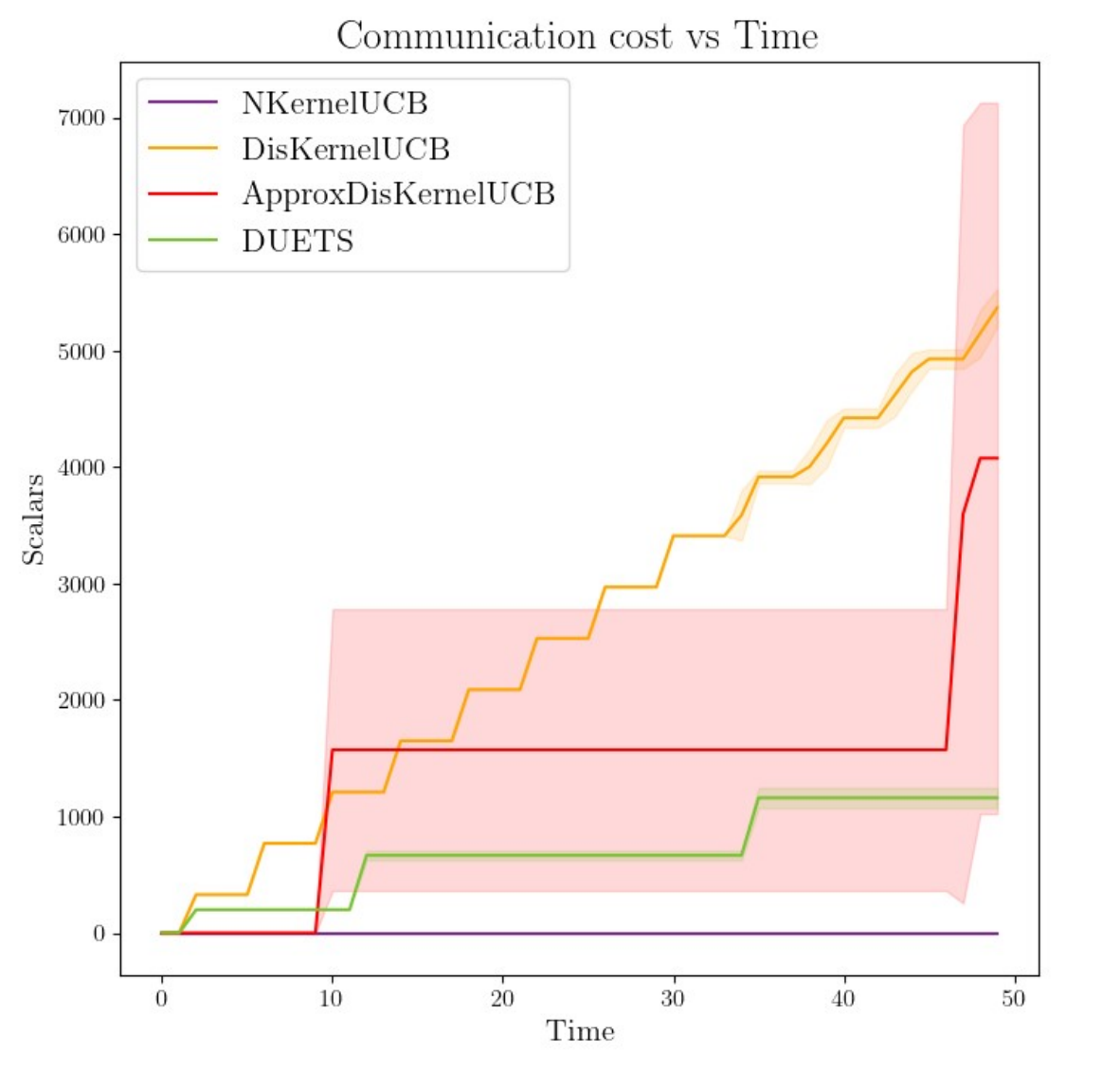}}
~
\subfloat[Branin]{\label{fig:branin_comm}\centering \includegraphics[scale = 0.2]{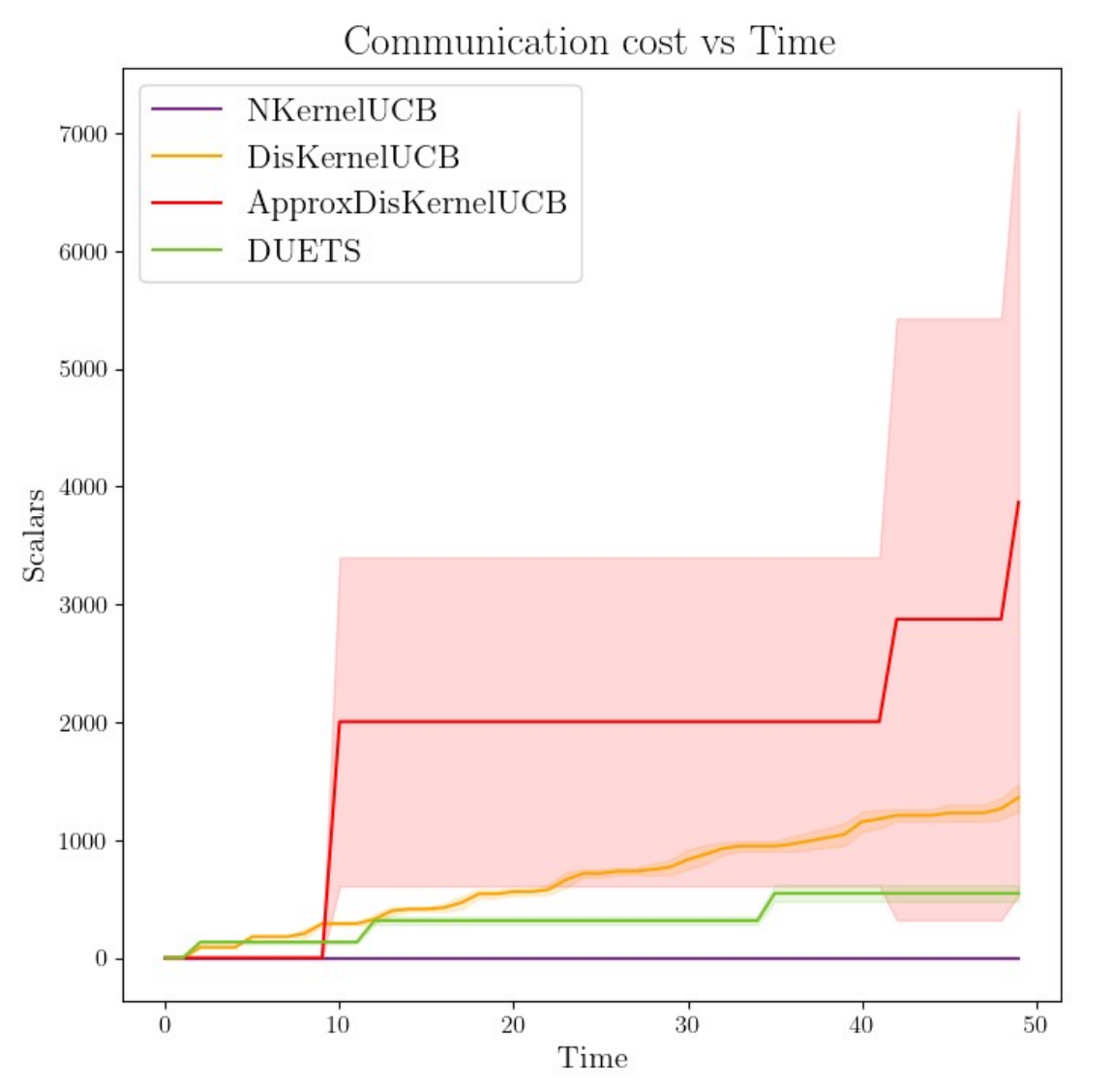}}
~
\subfloat[Hartmann-$4$D]{\label{fig:hartmann_comm}\centering \includegraphics[scale = 0.2]{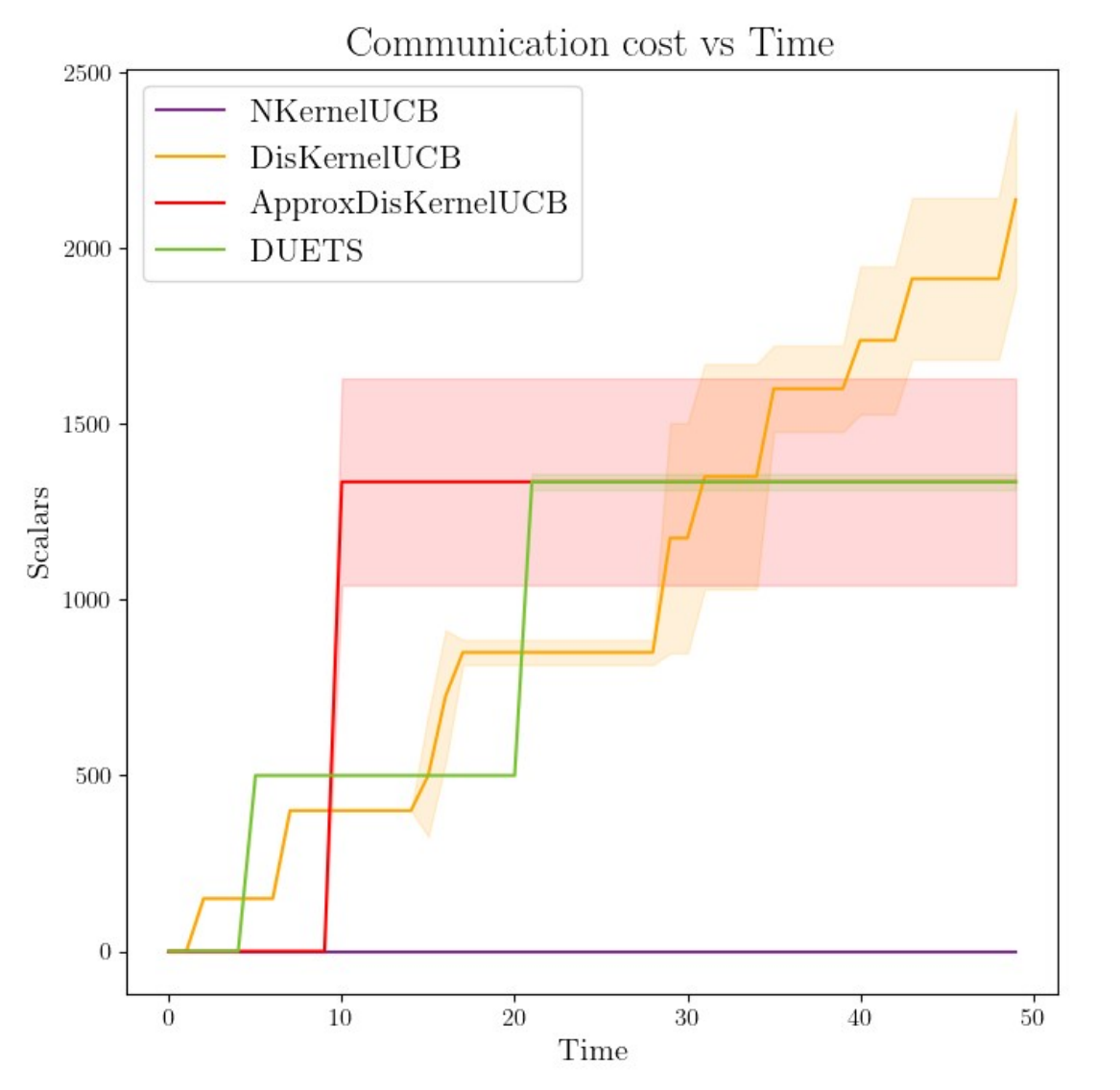}}
\caption{ Cumulative regret (Fig. (\ref{fig:cosine_regret}-\ref{fig:hartmann_regret}) and communication cost (\ref{fig:cosine_comm}-\ref{fig:hartmann_comm}) for all algorithms across different benchmark functions averaged over $5$ Monte Carlo runs. The shaded region represents error bars corresponding to one standard deviation. As seen from the above plots, \DUETS obtains a superior performance, both in terms of regret and communication cost, over other algorithm across all functions.}
\label{fig:plots_all}
\end{figure*}

\section{Empirical Studies}

We perform several empirical studies to corroborate our theoretical findings. We compare the regret performance and communication cost of our proposed algorithm, \textsc{DUETS}, against three baseline algorithms --- DisKernelUCB, ApproxDisKernelUCB and N-KernelUCB. The first two are distributed kernel bandits algorithms proposed in~\citet{Approx_Dis}. N-KernelUCB is a baseline algorithm considered in~\citet{Approx_Dis} where each agent locally runs the GP-UCB algorithm~\citep{Gopalan} with no communication among the agents. \\

We compare the performance of all the four algorithm across four benchmark functions. The first two are synthetic functions $h_1, h_2: \cB \to \R$ considered in~\citet{Approx_Dis}, where $\cB$ denotes the unit ball centered at origin in $\R^{10}$. The functions are given by:
\begin{align*}
    h_1(x) & := \cos(3x^{\top} \theta^{\star}) \\
    h_2(x) & := (x^{\top} \theta^{\star})^{3} - 3(x^{\top} \theta^{\star})^2 + 3(x^{\top} \theta^{\star}) + 3.
\end{align*}
For both the functions $\theta^{\star}$ is randomly chosen from the surface of the unit ball $\cB$. The other two functions are Branin~\citep{Azimi2012Branin, Picheny2013Branin} and Hartmann-$4$D~\citep{Picheny2013Branin}, which are commonly used benchmark functions for Bayesian Optimization. The Branin function is defined over $\cX = [0,1]^2$ while the Hartmann-$4$D function is defined over $\cX = [0,1]^{4}$. \\

We consider a distributed kernel bandit described in Section~\ref{sec:problem_formulation} with $N = 10$ agents. For all the experiments, we use the Squared Exponential kernel. The length scale was set to $0.2$ for Branin and to $1$ for all other functions. The observations were corrupted with zero mean Gaussian noise with a standard deviation of $0.2$. The parameter $D$ for ApproxDisKernelUCB and DisKernelUCB was set to $20$ and $10$ respectively. For \DUETS, we set $T_1 = 2$ and $p_0 = 10$. The parameter $\beta$ was selected using a grid search over $\{0.2, 0.5, 1, 2, 5\}$ for all the algorithms. All the algorithms were run for $T = 50$ time steps. We averaged the cumulative regret and the communication cost incurred by different algorithms over $5$ Monte Carlo runs. \\

The cumulative regret incurred by different algorithms across the different benchmark function are shown in the top row of Figure~\ref{fig:plots_all}. The bottom row consists of the corresponding plots for the communication cost incurred by the different algorithm. The shaded regions denotes error bars upto standard deviation on either side of the mean value. As evident from the plots, \DUETS achieves a significantly lower regret as compared to all other algorithms consistently across benchmark functions. \DUETS also incurs a smaller communication overhead as compared to other algorithms, corroborating our theoretical results.

\bibliography{sample}
\bibliographystyle{abbrvnat}

\appendix
% \onecolumn

\section{Appendix A.}\label{sec: Appendix A}

\subsection{Proof of Theorem~\ref{Main_Theorem}}

\begin{proof}
\label{appendix:proof_of_main_theorem}
In this section, we provide a detailed proof of Theorem~\ref{Main_Theorem}. For the regret bound, we first bound the regret incurred by \DUETS in each epoch $j$ and then sum it across different epochs to obtain a bound on the overall cumulative regret. We first prove the theorem assuming the results from Lemmas~\ref{grid_bound},~\ref{max_var_bound} and~\ref{lemma:epoch_number} and then separately prove the lemmas. \\

Consider any epoch $j \geq 1$ and let $R^{(j)}$ denote the regret incurred by \DUETS in this epoch. Since the agents take purely exploratory actions by uniform sampling points from the current set, we have the following crude bound $R^{(j)} \leq \Delta_j \cdot NT_j \cdot M_f$, where $\Delta_j := \sup_{x \in \cX_{j}} (f(x^*) - f(x))$. The term $NT_j \cdot M_f$ corresponds to number of points sampled during the epoch as we sample each connected component of $\cX_j$, of which there are at most $M_f$, $NT_j$ times. For $j = 1$, we use the trivial bound, 
\begin{align*}
    \Delta_1 = \sup_{x \in \cX} (f(x^*) - f(x)) \leq 2 \sup_{x \in \cX} f(x) \leq 2B,
\end{align*}
which gives us $R^{(1)} \leq 2B \cdot NT_1 \cdot M_f$. On invoking Lemma~\ref{grid_bound} for $j>1$ we obtain,
\begin{align*}
    R^{(j)} & \leq \Delta_j \cdot NT_j \cdot M_f \\
    & \leq NT_j \cdot M_f \cdot \left(8\beta(\delta') \cdot \left( \sup_{x\in \mathcal{X}_{j-1}}\sigma_{j-1}(x) \right) +\frac{4B}{T} \right),
\end{align*}
where $\delta' = \dfrac{\delta}{{2(\log\log{NT}+4) |\mathcal{U}_{T}|}}$. Using Lemma~\ref{max_var_bound}, we can further bound this expression as 
\begin{align*}
    R^{(j)} & \leq \Delta_j \cdot NT_j \cdot M_f \\
    & \leq NT_j \cdot M_f \cdot \left(8\beta(\delta') \cdot C_{f, \cX} \cdot \sqrt{\frac{\gamma_{NT_{j-1}}}{NT_{j-1}}} +\frac{4B}{T} \right) \\
    & \leq M_f \cdot \left(8C_{f, \cX}^{1/2} \cdot \beta(\delta')  \cdot \sqrt{NT\gamma_{NT_{j-1}}} +\frac{4B NT_j}{T} \right) \\
    % + 6C_{f, \cX}^{1/2} \cdot \beta(\delta')\cdot\sqrt{N\gamma_{NT_{j-1}}}
    & \leq M_f \cdot \left(8C_{f, \cX}^{1/2} \cdot \beta(\delta')  \cdot \sqrt{NT\gamma_{NT}} +\frac{4B NT_j}{T} \right).
    % + 6C_{f, \cX}^{1/2} \cdot \beta(\delta')\cdot\sqrt{N\gamma_{NT}}
\end{align*}

In the third line, we used the inequality $\frac{T_j}{\sqrt{T_{j-1}}}\leq \sqrt{T}$, which follows from the definition of $T_j$. In the last line, we used the fact that $\gamma_{m}$ is an increasing function of $m$. Thus, if $J$ denotes an upper bound on the number of epochs, we can  write:
\begin{align}
    \sum_{j = 1}^J R^{(j)} & \leq 2B M_f \cdot NT_1 + \sum_{j = 2}^J M_f \cdot \left(8C_{f, \cX}^{1/2} \cdot \beta(\delta')  \cdot \sqrt{NT\gamma_{NT}} +\frac{4B NT_j}{T} \right) \nonumber \\
    & \leq 2B M_f \cdot NT_1 + J \cdot \left(8C_{f, \cX}' \cdot \beta(\delta')  \cdot \sqrt{NT\gamma_{NT}} \right) + \frac{4B N M_f}{T} \sum_{j = 1}^J T_j   \nonumber \\
    & \leq 2B M_f \cdot NT_1 + J \cdot \left(8C_{f, \cX}' \cdot \beta(\delta')  \cdot \sqrt{NT\gamma_{NT}} \right) + {4B N M_f}.\label{eqn:regret_analysis_last} 
\end{align}

We next optimize the length of the first epoch $T_1$ in order to achieve order optimal regret. \textsc{DUETS} achieves order optimal regret for $N\leq\max(T,\gamma_{NT})$.\\

If $N<T$ we can choose $T_1=\sqrt{\frac{T}{N}} +\overline{M}(\delta^{'})$ where $\delta^{'}=\frac{\delta}{2(\log\log{NT}+4)}$. Left hand side of equation (\ref{eqn:regret_analysis_last}) can now be written as $\widetilde{\mathcal{O}}(\sqrt{NT\gamma_{NT}}\beta(\delta^{'}))\equiv \widetilde{\mathcal{O}}\left(\sqrt{NT\gamma_{NT}}\left(\log{\frac{T}{\delta}}\right)\right)$.\\
If $N\leq\gamma_{NT}$ we can fix $T_1=\sqrt{T}+\overline{M}(\delta^{'})$. We have $ NT_1\leq  \widetilde{O}(\sqrt{NT\gamma_{NT}})$ and the left hand-side is once again $\widetilde{\mathcal{O}}\left(\sqrt{NT\gamma_{NT}}\left(\log{\frac{T}{\delta}}\right)\right)$.
\end{proof}

Note that by Lemma \ref{lemma:epoch_number}, $J$ is upper bounded by $\log(\log{N}\log{T})+4$ and is thus $\widetilde{O}(1)$. \\

Before moving onto the proofs of Lemmas~\ref{grid_bound} and~\ref{lemma:epoch_number}, we state two auxiliary lemmas that will be useful for our analysis.

\begin{definition}
    Let $\cD = \{x_1, x_2, \dots, x_m\} \subset \cX$ be a collection $m$ points and $\cS$ be any subset of $\cD$. Let $\sigma_{\cD}^2(\cdot)$ denote the posterior variance corresponding to the points in $\cD$ and $\tilde{\sigma}_{\cS}^2(\cdot)$ denote the \emph{approximate} posterior computed based on the points in $\cS$. We call $\cS$ to be an $\varepsilon$-accurate inducing set if the following relations are true for all $x \in \cX$.
    \begin{align*}
        \frac{1 - \varepsilon}{1 + \varepsilon} \cdot \tilde{\sigma}_{\cS}^2(x) \leq \sigma_{\cD}^2(x) \leq \frac{1 + \varepsilon}{1 - \varepsilon}  \cdot \tilde{\sigma}_{\cS}^2(x).
    \end{align*} 
\end{definition}

\begin{lemma}[Adapted from~\citet{Calandriello_Sketching}]
    Let $\cD = \{x_1, x_2, \dots, x_m\} \subset \cX$ be a collection $m$ points and $\cS$ be a random subset of $\cD$ constructed by including each point with probability $p$, independent of other points. Then $\cS$ is an $\varepsilon$-accurate inducing set with 
    probability $\displaystyle 1 - 4m\exp\left( -\frac{3p\varepsilon^2}{8\sigma_{\max}^2} \right)$, where $\sigma_{\max}^2 = \sup_{x \in \cX} \sigma_{\cD}^2(x)$.
    \label{lemma:calandriello_inducing_set}
\end{lemma}

\begin{lemma}
    Let \DUETS be run with a choice of $p_0 = 72 \log(4NT/\delta')$. Then, for all epochs $j \geq 1$, the global inducing set $\cS_j$ is $0.5$-accurate with probability $1- \delta$.
    \label{lemma:accuracy_of_S_j}
\end{lemma}
\begin{proof}
    The statement is an immediate consequence of Lemma~\ref{lemma:calandriello_inducing_set} with the given choice of parameter $p_0$.
\end{proof}

We are now ready to prove Lemmas~\ref{grid_bound} and~\ref{lemma:epoch_number}.

\subsection{Proof of Lemma~\ref{grid_bound}}
\label{grid_bound_proof}
\begin{proof}

Consider any epoch $j \geq 2$ and let $x \in \cX_j$. Let $\Delta(x) := f(x^*) - f(x)$. We will obtain a bound on $\Delta(x)$ for any general $x$ in order establish the bound on $\Delta_j$. Using the discretization from Assumption~\ref{Grid} for $\cX_{j}$, we obtain,
\begin{align*}
    \Delta(x) & = f(x^*) - f(x) \\
    & \leq f(x^*) - f([x^*]_{\cU_T}) +  f([x^*]_{\cU_T}) - (f(x) - f([x]_{\cU_T})) -  f([x]_{\cU_T}) \\
    & \leq  f([x^*]_{\cU_T}) -  f([x]_{\cU_T}) + \frac{2B}{T}.
\end{align*}
Using the result from~\citet{Vakili_Suddep_Sketching}, we obtain the following high probability bound that holds with probability $1- \delta$:
\begin{align*}
    \Delta(x) & \leq  f([x^*]_{\cU_T}) -  f([x]_{\cU_T}) + \frac{2B}{T} \\
    & \leq  \tilde{\mu}_j([x^*]_{\cU_T}) + \beta(\delta') \tilde{\sigma}_j([x^*]_{\cU_T}) -  \tilde{\mu}_j([x]_{\cU_T}) + \beta(\delta') \tilde{\sigma}_j([x]_{\cU_T}) + \frac{2B}{T} \\
    & \leq  \tilde{\mu}_j(x^*) -  \tilde{\mu}_j(x) + \beta(\delta') \tilde{\sigma}_j([x^*]_{\cU_T}) + \beta(\delta') \tilde{\sigma}_j([x]_{\cU_T}) + \frac{4B}{T},
\end{align*}
where we again used Assumption~\ref{Grid} in the last step. We claim that $x^* \in \cX_{j-1}$ for all $j \geq 2$. Assuming this claim this true, we can bound the above expression as
\begin{align*}
    \Delta(x) & \leq  \sup_{x \in \cX_{j-1}} \tilde{\mu}_j(x') -  \tilde{\mu}_j(x) + \beta(\delta') \tilde{\sigma}_j([x^*]_{\cU_T}) + \beta(\delta') \tilde{\sigma}_j([x]_{\cU_T}) + \frac{4B}{T} \\
    & \leq  2 \beta(\delta') \sigma_{j, \max}  + \beta(\delta') \tilde{\sigma}_j([x^*]_{\cU_T}) + \beta(\delta') \tilde{\sigma}_j([x]_{\cU_T}) + \frac{4B}{T},
\end{align*}
where we used the update condition (Eqn.~\eqref{eqn:update_rule}) in the second step. Since $\cS_j$ is $0.5$-accurate (Lemma~\ref{lemma:accuracy_of_S_j}), we have $\tilde{\sigma}_j^2(x) \leq 3 \sigma_j^2(x) \leq 3\sigma_{j, \max}^2$. On plugging this back into the above equation, we obtain,
\begin{align*}
    \Delta(x) & \leq  8 \beta(\delta') \sigma_{j, \max}  + \frac{4B}{T}.
\end{align*}
The statement of the lemma follows by $\Delta_j = \sup_{x \in \cX_j} \Delta(x)$. \\

We prove our claim $x^* \in \cX_{j}$ for all $j \geq 1$ using induction. Clearly, $x^* \in \cX_1 = \cX$, by definition. Assume $x^* \in \cX_{j-1}$ for some $j \geq 2$. Fix an arbitrary  $ x \in \mathcal{X}_{j-1}$, from the confidence bound lemma we have:
\begin{align*}
 \mu_{j-1}(x)-\mu_{j-1}(x^*)\leq &  (f(x)-f(x^*))+\beta(\delta^{'})(\sigma_j(x)+\sigma_j(x^*)) \leq  2\sigma_{j-1.\mathrm{max}}(x),
\end{align*}
where the second inequality follows as $f(x)\leq f(x^*)$. As the inequality holds $\forall x \in \mathcal{X}_{j-1}$ we must have:
\begin{align*}
 \sup_{x \in \mathcal{X}_{j-1}}\mu_{j-1}(x)-\mu_{j-1}(x^*)\leq 2\sigma_{j-1.\mathrm{max}}(x)
\end{align*}
and thus indeed $x\in \mathcal{X}_j$.

\end{proof}

\subsection{Proof of Lemma~\ref{lemma:epoch_number}}
\label{proof:epoch_number}

We define $E(s) := \min\{j : T_j \geq T/4 \ | \ T_1 = s \}$. Note that $T_j$ is an increasing function of $j$. Since $T_{E(s)} \geq T/4$, we can conclude that $E(s) + 4$ is an upper bound on the number of epochs. Thus, we focus on bounding $E(s)$. We first show that $E(s)$ is a non-decreasing function of $s$. \\

To that effect, we claim that for $j \geq 2$ the epoch lengths satisfy the relation $T_j \geq T^{1-2^{-j + 1}} \cdot  T_1^{2^{-j+1}}$. This relation follows immediately using induction. For the base case, note that $T_2 \geq T^{1/2} \cdot T_1^{1/2}$, by definition. Assume that the relation holds for $j - 1$. Thus, 
\begin{align}
    T_j \geq T^{1/2} \cdot T_{j-1}^{1/2} \geq T^{1/2} \cdot T^{1-2^{-(j-1) + 1 - 1} } \cdot T_1^{2^{-(j-1) + 1 - 1}} \geq T^{1-2^{-j + 1}} \cdot  T_1^{2^{-j+1}}.
    \label{eqn:T_j_bound}
\end{align}
Since $T_j$'s are lower bounded by an increasing function of $T_1$, the number of epochs $E(s)$ is a non-increasing function of $s$. Since $T_1 \geq \frac{T}{N}$, $E\left(\frac{T}{N}\right)$ is an upper bound on the number of epochs for all choices of $T_1$. \\

Let $j^* = \max\{\log(\log(T)), \log(\log(N))\}$.  Using the above relation for $T_j$ from Eqn.~\eqref{eqn:T_j_bound} and the lower bound on $T_1$, we have,
\begin{align*}
    T_{j^*} \geq T \cdot{N^{-2^{1-j}}}= T\cdot\left(2^{-\frac{\log{N}}{2^{j}}}\right)^2\geq T\cdot2^{-2}
\end{align*}
We can  hence conclude that $T_{j^*} \geq T/4$, which implies that $E(T_1) \leq j^*$ for all permissible choices of $T_1$. Consequently, the number of epochs are bounded as $\log(\log(\max\{N, T\})) + 4$.

\subsection{Proof of Lemma~\ref{lemma:inducing_set_size}}

For all epochs $j \geq 1$, recall that the inducing set is constructed by including each point from $\cD_j$ with probability $p_j$, independent of other points. Thus, $|\cS_j|$ is a binomial random variable with parameters $|\cD_j| = NT_j$ and $p_j$. Using the Chernoff bound for Binomial random variables, we can conclude that 
\begin{align*}
    \Pr(|\mathcal{S}_j|>(1+\varepsilon) NT_j p_j)\leq \exp\left(-\frac{\varepsilon^2 NT_j p_j}{2+\varepsilon}\right).
\end{align*}
Invoking the bound with $\varepsilon = 2+\log(1/\delta')$, with $\delta' = \delta/(\log\log(NT) + 4)$ yields that the following relation holds with probability $1 - \delta'$:
\begin{align*}
    |\cS_j| & \leq (3 + \log(1/\delta')) \cdot NT_j p_j \\
    & \leq (3 + \log(1/\delta')) \cdot NT_j \cdot p_0 \sigma_{j, \max}^2 \\
    & \leq (3 + \log(1/\delta')) \cdot NT_j p_0 \cdot C_{f, \cX} \cdot \frac{\gamma_{NT_j}}{NT_j}  \\
    & \leq (3 + \log(1/\delta')) p_0 \gamma_{NT},
\end{align*}
where we used Lemma~\ref{max_var_bound} in the third step and monotonicity of $\gamma_m$ in the last step. On taking a union bound over all epochs and using the bound on the number of epochs from Lemma~\ref{lemma:epoch_number}, we conclude that for all epochs $j$, $|\cS_j| = \tilde{\cO}(\gamma_{NT})$ with probability $1 - \delta$.

% \bibliography{references}
% \bibliographystyle{abbrvnat}

% \newpage
% \appendix
% % \onecolumn
% \input{appendix}

\end{document}